\documentclass[twoside,11pt]{article}

\usepackage[accepted]{melba}

\usepackage{times}
\usepackage{epsfig}
\usepackage{amsmath,amssymb,amsfonts}
\usepackage[ruled,vlined]{algorithm2e}

\SetCommentSty{mycommfont}

\usepackage{enumerate}
\usepackage[inline]{enumitem}
\usepackage{url} % not crucial - just used below for the URL 
\usepackage{color}
\usepackage{multirow}
\usepackage{mathtools}
\usepackage{float}
\usepackage{wrapfig,lipsum,booktabs}
\usepackage{caption}
\usepackage{subcaption}

\usepackage{tikz}
\usepackage{tikz-cd}
\usetikzlibrary{arrows,positioning,calc} 
\tikzset{
	>=stealth',
    every node/.style={align=center},
	punkt/.style={
		   rectangle,
		   rounded corners,
		   draw=black, very thick,
		   text width=6.5em,
		   minimum height=2em,
		   text centered},
	pil/.style={
		   ->,
		   thick,
		   shorten <=2pt,
		   shorten >=2pt,},
	commutative diagrams/.cd,
}

\newtheorem{prop}{Proposition}

\DeclareMathOperator{\spn}{span}

\DeclareMathOperator{\tr}{tr}

\DeclareMathOperator*{\argmin}{arg\,min}

\newcommand{\gr}{\text{Gr}}
\newcommand{\st}{\text{St}}

\newcommand{\gl}{\text{GL}}
\newcommand{\bd}[1]{\boldsymbol{#1}}

\newcommand{\R}{\mathbb{R}}
\newcommand{\Cpx}{\mathbb{C}}
\newcommand{\mc}[1]{\mathcal{#1}}

\graphicspath{
{./figure/}
}

\melbaheading{2022:002}{https://www.melba-journal.org/papers/2022:002.html}{2021}{1-21}{09/2021}{03/2022}{Yang and Vemuri}{Information Processing in Medical Imaging (IPMI) 2021}{Aasa Feragen, Stefan Sommer, Julia Schnabel, Mads Nielsen}

\ShortHeadings{Nested Grassmanns for Dimensionality Reduction}{Yang and Vemuri}
\firstpageno{1}

\title{Nested Grassmannians for Dimensionality Reduction with Applications}

\author{\name Chun-Hao Yang \email chunhaoy@ntu.edu.tw \\  % start right after \author{, or there will be an extra space
	\addr Institute of Applied Mathematical Science, National Taiwan University, Taipei, Taiwan
	\AND
	\name Baba C.\ Vemuri \email vemuri@ufl.edu \\
	\addr Department of CISE, University of Florida, Gainesville, FL, USA
}

\begin{document}

% top matter
\maketitle

% abstract
\begin{abstract}

In the recent past, nested structures in Riemannian manifolds has been studied in the context of dimensionality reduction as an alternative to the popular principal geodesic analysis (PGA) technique, for example, the principal nested spheres. In this paper, we propose a novel framework for constructing a nested sequence of homogeneous Riemannian manifolds. Common examples of homogeneous Riemannian manifolds include the $n$-sphere, the Stiefel manifold, the Grassmann manifold and many others. In particular, we focus on applying the proposed framework to the Grassmann manifold, giving rise to the nested Grassmannians (NG). An important application in which Grassmann manifolds are encountered is planar shape analysis. Specifically, each planar (2D) shape can be represented as a point in the complex projective space which is a complex Grassmann manifold. Some salient features of our framework are: (i) it explicitly exploits the geometry of the homogeneous Riemannian manifolds and (ii) the nested lower-dimensional submanifolds need not be geodesic. With the proposed NG structure, we develop algorithms for the supervised and unsupervised dimensionality reduction problems respectively. The proposed algorithms are compared with PGA via simulation studies and real data experiments and are shown to achieve a higher ratio of expressed variance compared to PGA.

%Grassmann manifolds have been widely used to represent the geometry of feature spaces in a variety of problems in medical imaging and computer vision including but not limited to shape analysis, action recognition, subspace clustering and motion segmentation. For these problems, the features usually lie in a very high-dimensional Grassmann manifold and hence an appropriate dimensionality reduction technique is called for in order to curtail the computational burden. To this end, the Principal Geodesic Analysis (PGA), a nonlinear extension of the well known principal component analysis, is applicable as a general tool to many Riemannian manifolds. In this paper, we propose a novel framework for dimensionality reduction of data in Riemannian homogeneous spaces and then focus on the Grassman manifold which is an example of a homogeneous space. Our framework explicitly exploits the geometry of the homogeneous space yielding reduced dimensional nested sub-manifolds that need not be geodesic submanifolds and thus are more expressive. Specifically, we project points in a Grassmann manifold to an embedded lower dimensional Grassmann manifold. A salient feature of our method is that it leads to higher expressed variance compared to PGA which we demonstrate via synthetic and real data experiments. 

%We apply our method to planar shape analysis applied to Corpus Callosi shapes derived from the OASIS database, for which the shape space is characterized by a Grassmann manifold in the complex vector space.
\end{abstract}


\section{Introduction}\label{sec:intro}

Riemannian manifolds are often used to model the sample space in which features derived from the raw data encountered in many medical imaging applications live. Common examples include the diffusion tensors (DTs) in diffusion tensor imaging (DTI) \citep{basser1994mr}, the ensemble average propagator (EAP) \citep{callaghan1993principles}. Both DTs and EAP are used to capture the diffusional properties of water molecules in the central nervous system by  non-invasively imaging the tissue via diffusion weighted magnetic resonance imaging. In DTI, diffusion at a voxel is captured  by a DT, which is a $3 \times 3$ symmetric positive-definite matrix, whereas EAP is a probability distribution characterizing the local diffusion at a voxel, which can be parametrized as a point on the Hilbert sphere. Another example is the shape space used to represent shapes in shape analysis. There are many ways to represent a shape, and the most simple one is to use landmarks. For the landmark-based representation, the shape space is called Kendall's shape space \citep{kendall1984shape}. Kendall's shape space is in general a stratified space \citep{goresky1988stratified, feragen2014geometry}, but for the special case of planar shapes, the shape space is the complex projective space, which is a complex Grassmann manifold. The examples mentioned above are often high-dimensional: a DTI scan usually contains half a million DTs; the shape of the Corpus Callosum (which is used in our experiments) is represented by a several hundreds of boundary points in $\R^2$. Thus, in these cases, dimension reduction techniques, if applied appropriately, can benefit the subsequent statistical analysis.

For data on Riemannian manifolds, the most widely used dimensionality reduction method is the principal geodesic analysis (PGA) \citep{fletcher2004principal}, which generalizes the principal component analysis (PCA) to manifold-valued data. In fact, there are many variants of PGA. \citet{fletcher2004principal} proposed to find the geodesic submanifold of a certain dimension that maximizes the projected variance and computationally, it was achieved via a linear approximation, i.e., applying PCA on the tangent space at the intrinsic mean. This is sometimes referred to as the tangent PCA. Note that this approximation requires the data to be clustered around the intrinsic mean, otherwise the tangent space approximation to the manifold leads to inaccuracies. Later on, \citet{sommer2010manifold} proposed the Exact PGA (EPGA), which does not use any linear approximation. However, EPGA is computationally expensive as it requires two non-linear optimizations steps per iteration (projection to the geodesic submanifold and finding the new geodesic direction such that the loss of information is minimized). \citet{chakraborty2016efficient} partially solved this problem for manifolds with constant sectional curvature (spheres and hyperbolic spaces) by deriving closed form formulae for the projection. Other variants of PGA include but are not limited to sparse exact PGA~\citep{banerjee2017sparse}, geodesic PCA~\citep{huckemann2010intrinsic}, and probabilistic PGA~\citep{zhang2013probabilistic}. All these methods focus on projecting data to a \emph{geodesic submanifold} as in PCA where one projects data to a vector subspace. Instead, one can also project data to a submanifold that minimizes the reconstruction error without any further restrictions, e.g.\ being geodesic. This is the generalization of the principal curve~\citep{hastie1989principal} to Riemannian manifolds presented in~\citet{hauberg2016principal}.

Another feature of PCA is that it produces a sequence of nested vector subspaces. From this observation, \citet{jung2012analysis} proposed the \emph{principal nested spheres} (PNS) by embedding an $n-1$-sphere into an $n$-sphere, and the embedding is not necessarily isometric. Hence PNS is more general than PGA in that PNS is not required to be geodesic. Similarly, for the manifold $P_n$ of $n \times n$ SPD matrices, \cite{harandi2018dimensionality} proposed a geometry-aware dimension reduction by projecting data in $P_n$ to $P_m$ for some $m \ll n$. They also applied the nested $P_n$ for the supervised dimensionality reduction problem. \cite{damon2014backwards} considered a nested sequence of relations which determine a nested sequence of submanifolds that are not necessarily geodesic. They showed various examples, including Euclidean space and the $n$-sphere, depicting how the nested relations generalized PCA and PNS. However, for an arbitrary Riemannian manifold, it is not clear how to construct a nested submanifold. Another generalization of PGA was proposed by \cite{pennec2018barycentric}, called the exponential barycentric subspace (EBS). A $k$-dimensional EBS is defined as the locus of weighted exponential barycenters of $(k + 1)$ affinely independent reference points. The EBSs are naturally nested by removing or adding reference points.

Unlike PGA which can be applied to any Riemannian manifolds, the construction of the nested manifolds relies heavily on the geometry of the specific manifold, and there is no general principle for such a construction. All the examples described above (spheres and $P_n$) and several others such as the Grassmannian, Stiefel etc. are homogeneous Riemannian manifolds \citep{helgason1979differential}, which intuitively means that all points on the manifold 'look' the same. In this work, we propose a general framework for constructing a nested sequence of homogeneous Riemannian manifolds, and, via some simple algebraic computation, show that the nested sphere and the nested $P_n$ can indeed be derived within this framework. We will then apply this framework to the Grassmann manifolds, called the nested Grassmann manifolds (NG). The Grassmann manifold $\gr(p, \mathbb{V})$ is the manifold of all $p$-dimensional subspaces of the vector space $\mathbb{V}$ where $1 \leq p \leq \dim \mathbb{V}$. Usually $\mathbb{V} = \R^n$ or $\mathbb{V} = \Cpx^n$. An important example is Kendall's shape space of 2D shapes. The space of all shapes determined by $k$ landmarks in $\R^2$ is denoted by $\Sigma^k_2$, and \citet{kendall1984shape} showed that it is isomorphic to the complex projective space $\Cpx P^{k-2} \cong \gr(1, \Cpx^{k-1})$. In many applications, the number $k$ of landmarks is large, and so is the dimension of $\gr(1, \Cpx^{k-1})$. The core of the proposed dimensionality reduction involves projecting data on $\gr(p, \mathbb{V})$ to $\gr(p, \widetilde{\mathbb{V}})$ with $\dim \widetilde{\mathbb{V}} \ll \dim \mathbb{V}$. The main contributions of this work are as follows: (i) We propose a general framework for constructing a nested sequence of homogeneous Riemannian manifolds unifying the recently proposed nested spheres \citep{jung2012analysis} and nested SPD manifolds \citep{harandi2018dimensionality}. (ii) We present novel dimensionality reduction techniques based on the concept of NG in both supervised and unsupervised settings respectively. (iii) We demonstrate via several simulation studies and real data experiments, that the proposed NG can achieve a higher ratio of expressed variance compared to PGA.

The rest of the paper is organized as follows. In Section~\ref{sec:nested_homo}, we briefly review the definition of homogeneous Riemannian manifolds and present the recipe for the construction of nested homogeneous Riemannian manifolds. In Section~\ref{sec:nested_grassmannian}, we first review the geometry of the Grassmannian. By applying the procedure developed in Section~\ref{sec:nested_homo}, we present the nested Grassmann manifolds representation and discuss some of its properties in details. Then we describe algorithms for our unsupervised and supervised dimensionality reduction techniques, called the Principal Nested Grassmanns (PNG), in Section~\ref{sec:dr}. In Section~\ref{sec:exp}, we present several simulation studies and experimental results showing how the PNG technique performs compared to PGA under different settings. Finally, we draw conclusions in Section~\ref{sec:conc}.

\section{Nested Homogeneous Spaces}\label{sec:nested_homo}
In this section, we introduce the structure of nested homogeneous Riemannian manifolds. A Riemannian manifold $(M, \tau)$ is \emph{homogeneous} if the group of isometries $G = I(M)$ admitted by the manifold acts transitively on $M$ \citep{helgason1979differential}, i.e., for $x, y \in M$, there exists $g \in G$ such that $g(x) = y$. In this case, $M$ can be identified with $G/H$ where $H$ is an isotropy subgroup of $G$ at some point $p \in M$, i.e.\ $H = \{g \in G: g(p) = p\}$. Examples of homogeneous Riemannian manifolds include but are not limited to, the $n$-spheres $S^{n-1} = \text{SO}(n)/\text{SO}(n-1)$, the SPD manifolds $P_n = \text{GL}(n)/\text{O}(n)$, the Stiefel manifolds $\text{St}(m,n) = \text{SO}(n)/\text{SO}(n-m)$, and the Grassmann manifolds $\gr(p,n) = \text{SO}(n)/S(\text{O}(p) \times \text{O}(n-p))$.

In this paper, we focus on the case where $G$ is either a real or a complex matrix Lie group, i.e.\ $G$ is a subgroup of $\text{GL}(n, \R)$ or $\text{GL}(n, \Cpx)$. The main idea behind the construction of nested homogeneous spaces is simple: \emph{augmenting the matrix in $G$ in an appropriate way}. With an embedding of the isometry group $G$, the embedding of the homogeneous space $G/H$ follows naturally from the quotient structure.

Let $G$ and $\tilde{G}$ be two connected Lie groups such that $\dim G < \dim \tilde{G}$ and $\tilde{\iota}: G \to \tilde{G}$ be an embedding. For a closed connected subgroup $H$ of $G$, let $\tilde{H} = \tilde{\iota}(H)$. Since $\tilde{\iota}$ is an embedding, $\tilde{H}$ is also a closed subgroup of $\tilde{G}$. Now the canonical embedding of $G/H$ in $\tilde{G}/\tilde{H}$ is defined by $\iota(gH) = \tilde{\iota}(g)\tilde{H}$ for $g \in G$. It is easy to see that $\iota$ is well-defined. Let $g_1, g_2 \in G$ be such that $g_1 = g_2h$ for some $h \in H$. Then
\[
    \iota(g_1H) = \tilde{\iota}(g_1)\tilde{H} = \tilde{\iota}(g_2h)\tilde{H} = \tilde{\iota}(g_2)\tilde{\iota}(h)\tilde{H} = \tilde{\iota}(g_2)\tilde{H} = \iota(g_2H).
\]

Now for the homogeneous Riemannian manifolds $(M = G/H, \tau_1)$ and $(\tilde{M} = \tilde{G}/\tilde{H}, \tau_2)$, denote the left-$G$-invariant, right-$H$-invariant metric on $G$ (resp.\ left-$\tilde{G}$-invariant, right-$\tilde{H}$-invariant metric on $\tilde{G}$) by $\bar{\tau}_1$ and $\bar{\tau}_2$, respectively (see \citet[Prop. 3.16(4)]{cheeger1975comparison}).
\begin{prop}\label{prop:isometry_homo}
If $\tilde{\iota}: G \to \tilde{G}$ is isometric, then so is $\iota: G/H \to \tilde{G}/\tilde{H}$.
\end{prop}
\begin{proof}
Denote the Riemannian submersions by $\pi:G \to G/H$ and $\tilde{\pi}:\tilde{G} \to \tilde{G}/\tilde{H}$. Let $X$ and $Y$ be vector fields on $G/H$ and $\bar{X}$ and $\bar{Y}$ be their horizontal lifts respectively, i.e.\ $d\pi(\bar{X}) = X$ and $d\pi(\bar{Y}) = Y$. By the definition of Riemannian submersions, $d\pi$ is isometric on the horizontal spaces, i.e.\ $\bar{\tau}_1(\bar{X}, \bar{Y}) = \tau_1(d\pi(\bar{X}), d\pi(\bar{Y})) = \tau_1(X, Y)$. Since $\tilde{\iota}$ is isometric, we have $\bar{\tau}_1(\bar{X}, \bar{Y}) = \bar{\tau}_2(d\tilde{\iota}(\bar{X}), d\tilde{\iota}(\bar{Y}))$. By the definition of $\iota$, we also have $\iota \circ \pi = \tilde{\pi} \circ \tilde{\iota}$, which implies $d\iota \circ d\pi = d\tilde{\pi} \circ d\tilde{\iota}$. Hence,
\begin{align*}
\tau_1(X, Y) & = \bar{\tau}_1(\bar{X}, \bar{Y}) = \bar{\tau}_2(d\tilde{\iota}(\bar{X}), d\tilde{\iota}(\bar{Y}))\\
& = \tau_2((d\tilde{\pi}\circ d\tilde{\iota})(\bar{X}), (d\tilde{\pi}\circ d\tilde{\iota})(\bar{Y}))\\
& = \tau_2((d\iota\circ d\pi)(\bar{X}), (d\iota\circ d\pi)(\bar{Y}))\\
& = \tau_2(d\iota(X), d\iota(Y))
\end{align*}
where the third equality follows from the isometry of $d\tilde{\pi}$.
\end{proof}

Proposition~\ref{prop:isometry_homo} simply says that if the isometry group is isometrically embedded, then the associated homogeneous Riemannian manifolds will also be isometrically embedded. Conversely, if we have a Riemannian submersion $\tilde{f}: \tilde{G} \to G$, it can easily be shown that the induced map $f: \tilde{G}/\tilde{H} \to G/H$ would also be a Riemannian submersion where $H = \tilde{f}(\tilde{H})$. The construction above can be applied to a sequence of homogeneous spaces $\{M_m\}_{m=1}^\infty$, i.e.\ the embedding $\iota_m:M_m \to M_{m+1}$ can be induced from the embedding of the isometry groups $\tilde{\iota}_m: G_m \to G_{m+1}$ where $G_m = I(M_m)$ provided $\dim G_i < \dim G_j$ for $i < j$. See Figure~\ref{fig:homo_diagram} for the structure of nested homogeneous spaces.  

\begin{figure}[H]
    \centering
    \begin{tikzcd}[row sep=large, font=\normalsize]
    $G_m$ \arrow[hookrightarrow]{r}{\Large $\tilde{\iota}$}
        \arrow[rightarrow]{d}{\Large $\pi$}
      & $G_{m+1}$ \arrow[rightarrow]{d}{\Large $\pi$}\\
    $G_m/H_m$ \arrow[leftrightarrow]{d}{\Large $f$} 
      & $G_{m+1}/H_{m+1}$\arrow[leftrightarrow]{d}{\Large $f$}\\
    $M_m$ \arrow[hookrightarrow]{r}{\Large $\iota$}
      & $M_{m+1}$
    \end{tikzcd}
    \caption{Commutative diagram of the induced embedding for homogeneous spaces.}
    \label{fig:homo_diagram}
\end{figure}

\section{Nested Grassmann Manifolds}\label{sec:nested_grassmannian}

In this section, we will apply the theory of nested homogeneous space from the previous section to the Grassmann manifolds. First, we briefly review the geometry of the Grassmann manifolds in Section~\ref{sub:geometry}. With the theory in Section~\ref{sec:nested_homo}, we derive the nested Grassmann manifolds in Section~\ref{sub:embedding}, and the derivation for nested spheres and nested SPD manifolds are carried out in Section~\ref{sub:nested_examples}. 

\subsection{The Riemannian Geometry of Grassmann Manifolds}%
\label{sub:geometry}

To simplify the notation, we assume $\mathbb{V} = \R^n$ and write $\gr(p, n) \coloneqq \gr(p, \R^n)$. All the results presented in this section can be easily extended to the case of $\mathbb{V} = \Cpx^n$ by replacing transposition with conjugate transposition and orthogonal groups with unitary groups. The Grassmann manifold $\gr(p,n)$ is the manifold of all $p$-dimensional subspaces of $\R^n$, and for a subspace $\mc{X} \in \gr(p,n)$, we write $\mc{X} = \spn(X)$ where the columns of $X$ form an orthonormal basis for $\mc{X}$. The space of all $n \times p$ matrices $X$ such that $X^TX = I_p$ is called the \emph{Stiefel manifold}, denoted by $\st(p,n)$. Special cases of Stiefel manifolds are the Lie group of all orthogonal matrices, $\text{O}(n) = \st(n,n)$, and the $n$-sphere, $S^{n-1} = \st(1,n)$. The Stiefel manifold with the induced Euclidean metric (i.e.\ for $U, V \in T_X\st(p,n)$, $\langle U, V \rangle_X = \tr(U^TV)$) is a homogeneous Riemannian manifold, $\st(p,n) = \text{SO}(n)/\text{SO}(n-p)$. A canonical Riemannian metric on the Grassmann manifold can be inherited from the metric on $\st(p, n)$ since it is invariant to the left multiplication by elements of $O(n)$ \citep{absil2004riemannian,edelman1998geometry}. The Grassmann manifold with this metric is also homogeneous, $\gr(p, n) = \text{SO}(n)/S(\text{O}(p) \times \text{O}(n-p))$. 

With this canonical metric on the Grassmann manifolds, the geodesic can be expressed in closed form. Let $\mc{X}=\spn(X) \in \gr(p,n)$ where $X \in \st(p, n)$ and $H$ be an $n \times p$ matrix. Then the geodesic $\gamma(t)$ with $\gamma(0) = \mc{X}$ and $\gamma^{\prime}(0) = H$ is given by $\gamma_{\mc{X}, H}(t) = \spn(XV \cos \Sigma t + U \sin \Sigma t)$ where $H = U\Sigma V^T$ is the compact singular value decomposition \citep[Theorem 2.3]{edelman1998geometry}. The \emph{exponential map} at $\mc{X}$ is a map from $T_{\mc{X}}\gr(p, n)$ to $\gr(p, n)$ defined by $\text{Exp}_{\mc{X}}H = \gamma_{\mc{X}, H}(1) = \spn(XV\cos \Sigma + U\sin \Sigma)$. If $X^TY$ is invertible, the geodesic distance between $\mc{X} = \spn(X)$ and $\mc{Y} = \spn(Y)$ is given by $d_g^2(\mc{X}, \mc{Y}) = \tr \Theta^2 = \sum_{i=1}^p\theta_i^2$ where $(I - XX^T)Y(X^TY)^{-1} = U\Sigma V^T$, $U\in\st(p, n)$, $V \in \text{O}(p)$, and $\Theta = \tan^{-1}\Sigma$.  The diagonal entries $\theta_1,\ldots,\theta_k$ of $\Theta$ are known as the principal angles between $\mc{X}$ and $\mc{Y}$. 

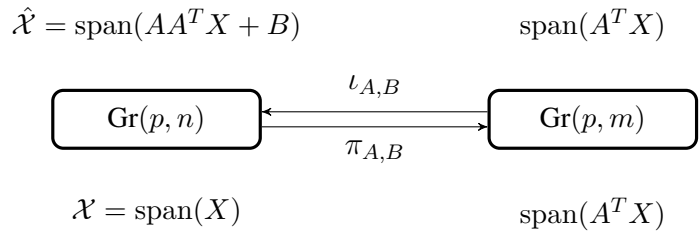
\begin{figure}[H]
    \centering
    \begin{tikzpicture}[node distance=3cm, auto,]
        \node[punkt] (space1) {$\gr(p, n)$};
        \node[punkt, right=of space1] (space2) {$\gr(p, m)$};
        \path[->] ($(space1.east)+(0,-0.1)$) edge node[below=0.5mm] {\large $\pi_{A,B}$} ($(space2.west)+(0,-0.1)$);
        \path[<-] ($(space1.east)+(0,+0.1)$) edge node[above=0.5mm] {\large $\iota_{A,B}$} ($(space2.west)+(0,+0.1)$);
        \node[below=0.5cm of space1] (X) {$\mc{X}=\spn(X)$};
        \node[above=0.5cm of space1, xshift=0cm] (Xrecon) {$\hat{\mc{X}} = \spn(AA^TX+B)$};
        \node[below=0.5cm of space2] (Xproj) {$\spn(A^TX)$};
        \node[above=0.5cm of space2] (Xproj) {$\spn(A^TX)$};
    \end{tikzpicture}
    \caption{Illustration of the embedding of $\gr(p, m)$ in $\gr(p, n)$ parametrized by $A \in \st(m, n)$ and $B \in \R^{n \times p}$ such that $A^TB = 0$.}
    \label{fig:diagram}
\end{figure}

\subsection{Embedding of \texorpdfstring{$\gr(p, m)$}{Lg} in \texorpdfstring{$\gr(p, n)$}{Lg}}%
\label{sub:embedding}

Let $\mc{X} = \spn(X) \in \gr(p, m)$, $X \in \st(p,m)$. The map $\iota:\gr(p, m) \to \gr(p, n)$, for $m < n$, defined by 
\[
 \iota(\mc{X})=\spn\left(\left[\begin{array}{c}
                X\\0_{(n-m)\times p}
            \end{array}\right]\right) 
\]
is an embedding and it is easy to check that this embedding is isometric \citep[Eq.\ (8)]{ye2016schubert}. However, for the dimensionality reduction problem, the above embedding is insufficient as it is not flexible enough to encompass other possible embeddings. To design flexible embeddings, we apply the framework proposed in Section~\ref{sec:nested_homo}. We consider $M_m = \gr(p,m)$ for which the isometry groups are $G_m = \text{SO}(m)$ and $H_m = \text{S}(\text{O}(p) \times \text{O}(m-p))$.

In this paper, we consider the embedding $\tilde{\iota}_m: \text{SO}(m) \to \text{SO}(m+1)$ given by,
\begin{align} \label{eq:embedding_ortho}
        \tilde{\iota}_m(O) = \text{GS}\left(R \left[\begin{array}{cc}
                O & a\\
                b^T & c
            \end{array}\right]\right)
\end{align}
where $O \in \text{SO}(m)$, $R \in \text{SO}(m+1)$, $a, b \in \R^m$, $c \in \R$, $c \neq b^TO^{-1}a$, and $\text{GS}(\cdot)$ is the Gram-Schmidt process. Since the Riemannian submersion $\pi:\text{SO}(m) \to \gr(p,m)$ is defined by $\pi(O) = \spn(O_p)$ where $O \in \text{SO}(m)$ and $O_p$ is the $m\times p$ matrix containing the first $p$ columns of $O$, the induced embedding $\iota_m:\gr(p, m) \to \gr(p, m+1)$ is given by,
\[
\iota_m(\mc{X}) = \spn\left(R\left[\begin{array}{c}
                X\\
                b^T
            \end{array}\right]\right) = \spn(\tilde{R}X + vb^T) ,
\]
where $b \in \R^p$, $R \in \text{SO}(m+1)$, $\tilde{R}$ contains the first $m$ columns of $R$ (which means $\tilde{R} \in \st(m, m+1)$), $v$ is the last column of $R$, and $\mc{X} = \spn(X) \in \gr(p,m)$. It is easy to see that for $R = I$ and $b=0$, this gives the natural embedding described in \citet{ye2016schubert} and at the beginning of this section.
\begin{prop}\label{prop:isometry}
If $b = 0$, then $\iota_m$ is an isometric embedding.
\end{prop}
\begin{proof}
With Proposition~\ref{prop:isometry_homo}, it suffices to show that $\tilde{\iota}_m$ is isometric when $b = 0$. Note that as $\iota_m$ is independent of $a$ and $c$ in the definition of $\tilde{\iota}_m$, we can assume $a = 0$ and $c=1$ without loss of generality. If $b = 0$, $\tilde{\iota}_m$ simplifies to 
\[
        \tilde{\iota}_m(O) = R \left[\begin{array}{cc}
                O & 0\\
                0 & 1
            \end{array}\right]
\]
where $R \in \text{SO}(m+1)$. The Riemannian distance on $\text{SO}(n)$ given the induced Euclidean metric is $d_{\text{SO}}(O_1, O_2) = \frac{1}{\sqrt{2}}\|\log O_1^TO_2\|_F$. Then for $O_1, O_2 \in \text{SO}(m)$, 
\[
    d_{\text{SO}}(\tilde{\iota}_m(O_1), \tilde{\iota}_m(O_2)) = \frac{1}{\sqrt{2}}\left\|\log\left(
    \left[\begin{array}{cc}
                O_1^TO_2 & 0\\
                0 & 1
            \end{array}\right]
    \right)\right\|_F = d_{\text{SO}}(O_1, O_2).
\]
Therefore $\tilde{\iota}_m$ is an isometric embedding, and so is $\iota_m$ by Proposition~\ref{prop:isometry_homo}.
\end{proof}

With the embedding $\iota_m$, we can construct the corresponding projection $\pi_m: \gr(p, m+1) \to \gr(p, m)$ using the following proposition.
\begin{prop}\label{prop:proj}
The projection $\pi_m: \gr(p, m+1) \to \gr(p, m)$ corresponding to $\iota_m(\mc{X}) = \spn(\tilde{R}X + vb^T)$ is given by $\pi_m(\mc{X}) = \spn(\tilde{R}^TX)$. 
\end{prop}
\begin{proof}
First, let $\mc{Y} = \spn(Y) \in \gr(p, m)$ and $\mc{X} = \spn(X) \in \gr(p, m+1)$ be such that $\mc{X} = \spn(\tilde{R}Y + vb^T)$. Then $XL = \tilde{R}Y + vb^T$ for some $L \in \gl(p)$. Therefore, $Y = \tilde{R}^T(XL - vb^T) = \tilde{R}^TXL$ and $\mc{Y} = \spn(Y) = \spn(\tilde{R}^TXL) = \spn(\tilde{R}^TX)$. Hence, the projection is given by $\pi_m(\mc{X}) = \spn(\tilde{R}^TX)$. This completes the proof.
\end{proof}

Note that for $\mc{X} = \spn(X) \in \gr(p, m+1)$, $\iota_m(\pi_m(\mc{X})) = \spn(RR^TX + vb^T) = \spn((I-vv^T)X + vb^T)$ where $v \in \R^{m+1}$ and $\|v\|=1$. The nested relation can be extended inductively and we refer to this construction as the \emph{nested Grassmann structure}:
\[
\gr(p, m) \stackrel{\iota_m}{\hookrightarrow} \gr(p, m+1)
\stackrel{\iota_{m+1}}{\hookrightarrow} \ldots \stackrel{\iota_{n-2}}{\hookrightarrow}
\gr(p, n-1) \stackrel{\iota_{n-1}}{\hookrightarrow} \gr(p, n).
\]
Thus the embedding from $\gr(p, m)$ into $\gr(p, n)$ can be constructed inductively by $\iota \coloneqq \iota_{n-1} \circ \ldots \circ \iota_{m-1} \circ \iota_m$ and similarly for the corresponding projection. The explicit forms of the embedding and the projection are given in the following proposition.
\begin{prop}
The embedding of $\gr(p, m)$ into $\gr(p, n)$ for $m < n$ is given by $\iota_{A,B}(\mc{X}) = \spn(AX + B)$ where $A \in \st(m, n)$ and $B \in \R^{n \times p}$ such that $A^TB = 0$. The corresponding projection from $\gr(p, n)$ to $\gr(p, m)$ is given by $\pi_{A} = \spn(A^TX)$.
\end{prop}
\begin{proof}
By the definition, $\iota \coloneqq \iota_{n-1} \circ \ldots \circ \iota_{m-1} \circ \iota_m$ and thus
the embedding $\iota: \gr(p, m) \to \gr(p, n)$ can be simplified as
\[
\iota_{A,B}(\mc{X}) = \spn\left(\Bigg(\prod_{i=m}^{n-1}R_i\Bigg)X + \sum_{i=m}^{n-1} \Bigg(\prod_{j=i+1}^{n-1}R_j\Bigg)v_ib_i^T\right) = \spn(AX + B)
\]
where $R_i \in \st(i, i+1)$, $v_i$ is such that $[R_i\; v_i] \in \text{O}(i+1)$, $b_i \in \R^p$, $A = R_{n-1}R_{n-2}\cdots R_m \in \st(m, n)$, and $B = \sum_{i=m}^{n-1} \big(\prod_{j=i+1}^{n-1}R_j\big)v_ib_i^T$ is an $n \times p$ matrix. It is easy to see that $A^TB = 0$. Similar to Proposition~\ref{prop:proj}, the projection $\pi_{A}: \gr(p, n) \to \gr(p, m)$ is then given by $\pi_{A}(\mc{X}) = \spn(A^TX)$. This completes the proof.
\end{proof}

From Proposition~\ref{prop:isometry}, if $B = 0$ then $\iota_{A}$ is an isometric embedding. 
\emph{Hence, our nested Grassmann structure is more flexible than PGA as it allows one to project the data onto a non-geodesic submanifold.} An illustration is shown in Figure~\ref{fig:diagram}. The results discussed in this subsection can be generalized to any homogeneous space in principle. For a given homogeneous space, once the embedding of the groups of isometries (e.g., Eq.~\eqref{eq:embedding_ortho}) is determined, the induced embedding and the corresponding projection can be derived akin to the case of Grassmann manifolds.

\subsection{Connections to Other Nested Structures}\label{sub:nested_examples}

The nested homogeneous spaces proposed in this work (see Figure~\ref{fig:homo_diagram}) actually provides a unified framework within which, the nested spheres \citep{jung2012analysis} and the nested SPD manifolds \citep{harandi2018dimensionality} are special cases.

{\bf The $n$-Sphere Example:} Since the $n$-sphere can be identified with a homogeneous space $S^{n-1} \cong \text{O}(n)/\text{O}(n-1)$, with the embedding~\eqref{eq:embedding_ortho}, the induced embedding of $S^{n-1}$ into $S^n$ is 
\[
\iota(\bd{x}) = \text{GS}\left(R\left[\begin{array}{c}
                x\\
                b
            \end{array}\right]\right) = \frac{1}{\sqrt{1+b^2}}R\left[\begin{array}{c}
                x \\
                b
            \end{array}\right] = R\left[\begin{array}{c}
                \sin(r)x \\
                \cos(r)
            \end{array}\right]
\]
where $x \in S^{n-1}$, $b \in \R$, and $r = \cos^{-1}\big( \frac{b}{\sqrt{1+b^2}}\big)$. This is precisely the nested sphere proposed in \citet[Eq.\ (2)]{jung2012analysis}. 

{\bf The SPD Manifold Example:} For the $m$-dimensional SPD manifold denoted by $P_m$, $G_{m}=\gl(m)$ and $H_{m}=\text{O}(m)$. Consider the embedding $\tilde{\iota}_{m}:\gl(m) \to \gl(m+1)$ given by
\[
  \tilde{A} = \tilde{\iota}_m(A) = R \left[\begin{array}{cc}
          A & 0\\
          0 & 1
      \end{array}\right]R^{T},
\]
where $A \in \gl(m)$, $R \in \text{O}(m+1)$ and the corresponding projection $\tilde{\pi}_{m}:\gl(m+1) \to \gl(m)$ is 
\[
  \tilde{\pi}_{m}(\tilde{A}) = W^{T}\tilde{A}W
\]
where $W$ contains the first $m$ columns of $R = [W\; v]\in \text{O}(m+1)$ (i.e., $W \in \st(m, m+1)$ and $W^{T}v=0$). The submersion $\psi \circ f: \gl(m) \to P_{m}$ is given by $\psi\circ f(A) = A^{T}A$. Hence the induced embedding $\iota_{m}:P_{m}\to P_{m+1}$ and projection $\pi_{m}:P_{m+1}\to P_{m}$ are
\[
\iota_{m}(X)=WXW^{T}+vv^{T} \text{ and } \pi_{m}(X)=W^{T}XW
\]
which is the projection map used in \citet[Eq.\ (13)]{harandi2018dimensionality}. However, \cite{harandi2018dimensionality} did not perform any embedding or construct a nested family of SPD manifolds. Recently, it came to our attention that \cite{jaquier2020high} derived a similar nested family of SPD manifolds based on the projection maps in \cite{harandi2018dimensionality} described above.

\section{Dimensionality Reduction with Nested Grassmanns}\label{sec:dr}

In this section, we discuss how to apply the nested Grassmann structure to the problem of dimension reduction. In Section~\ref{sub:dr} and~\ref{sub:supervised_dr}, we describe the unsupervised and supervised dimension reduction based on the nested Grassmann manifolds. In Section~\ref{sub:distance}, we will discuss the choice of distance metrics required by the dimensionality reduction algorithm and present some technical details regarding the implementation. Analysis of principal nested Grassmann (PNG) will be introduced and discussed in Section~\ref{sub:nested_grassmann} and Section~\ref{sub:score}.

\subsection{Unsupervised Dimensionality Reduction}\label{sub:dr}

We can now apply the nested Grassmann structure to the problem of unsupervised dimensionality reduction. Suppose that we are given the points, $\mc{X}_1,\ldots,\mc{X}_N \in \gr(p, n)$. We would like to have lower dimensional representations in $\gr(p,m)$ for $\mc{X}_1,\ldots,\mc{X}_N$ with $m \ll n$. The desired projection map $\pi_{A}$ that we seek is obtained by the minimizing the reconstruction error, i.e.\ 
\[
L_u(A, B) = \frac{1}{N}\sum_{i=1}^N d^2(\mc{X}_i, \hat{\mc{X}}_i) = \frac{1}{N}\sum_{i=1}^N d^2(\spn(X_i),\spn(AA^TX_i + B))
\] 
where $d(\cdot,\cdot)$ is a distance metric on $\gr(p, n)$. It is clear that $L_u$ has a $\text{O}(m)$-symmetry in the first argument, i.e.\ $L_u(AO, B) = L_u(A, B)$ for $O \in \text{O}(m)$. Hence, the optimization is performed over the space $\st(m, n)/\text{O}(m) \cong \gr(m, n)$ when optimizing with respect to this particular loss function. Now we can apply the Riemannian gradient descent algorithm \citep{edelman1998geometry} to obtain $A$ and $B$ by optimizing $L_u(A, B)$ over $\spn(A) \in \gr(m, n)$ and $B \in \R^{n \times p}$ such that $A^TB = 0$. Note that the restriction $A^TB = 0$ simply means that the columns of $B$ are in the null space of $A^T$, denoted $N(A^T)$. Hence in practice this restriction can be handled as follows. For arbitrary $\tilde{B} \in \R^{n \times p}$, project $\tilde{B}$ on to $N(A^T)$, i.e.\ $B = P_{N(A^T)} \tilde{B}$ where $P_{N(A^T)} = I - AA^T$ is the projection from $\R^n$ to $N(A^T)$. Thus, the loss function can be written as   
\[
    L_u(A, B) = \frac{1}{N}\sum_{i=1}^N d^2(\spn(X_i),\spn(AA^TX_i + (I - AA^T)B))
\]
and it is optimized over $\gr(m, n) \times \R^{n \times p}$. Note that since we represent a subspace by its orthonormal basis, when $m > n/2$, we can use the isomorphism $\gr(m, n) \cong \gr(n-m, n)$ to further reduce the computational burden. This will be particularly useful when $m = n-1$ as in Section~\ref{sub:nested_grassmann}. Under this isomorphism $\gr(m, n) \cong \gr(n-m, n)$, the corresponding subspace of $\spn(A) \in \gr(m, n)$ is $\spn(A_{\perp}) \in \gr(n-m, n)$ where $A_{\perp}$ is an $n\times (n-m)$ matrix such that $[A \; A_{\perp}]$ is an orthogonal matrix. Hence the loss function $L_u$ can alternatively be expressed as 
\[
    L_u(A, B) = \frac{1}{N}\sum_{i=1}^N d^2(\spn(X_i),\spn((I-A_{\perp}A_{\perp}^T)X_i + A_{\perp}A_{\perp}^TB)).
\]

\begin{remark} 
The reduction of the space of all possible projections from $\st(m,n)$ to $\gr(m,n)$ is a consequence of the choice of the loss function $L_u$. With a different loss function, one might have a different space of possible projections.
\end{remark}

\subsection{Supervised Dimensionality Reduction}\label{sub:supervised_dr}

If in addition to $\mc{X}_1,\ldots,\mc{X}_N \in \gr(p, n)$, we are given the associated labels $y_1,\ldots,y_N \in \{1,\ldots,k\}$, then we would like to use this extra information to sharpen the result of dimensionality reduction. Specifically, we expect that after reducing the dimension, points from the same class preserve their proximity    while points from different classes are well separated. We use an \emph{affinity function} $a: \gr(p,n) \times \gr(p, n) \to \R$ to encode the structure of the data as suggested by \citet[Sec 3.1, Eq.\ (14)-(16)]{harandi2018dimensionality}. For completeness, we repeat the definition of the affinity function here. The affinity function is defined as $a(\mc{X}_i, \mc{X}_j) = g_w(\mc{X}_i, \mc{X}_j) - g_b(\mc{X}_i, \mc{X}_j)$ where
\begin{align*}
    g_{w}(\mc{X}_i,\mc{X}_j)& =\begin{cases}
1 & \text{if }\mc{X}_i\in N_{w}(\mc{X}_j)\text{ or }\mc{X}_j\in N_{w}(\mc{X}_i)\\
0 & \text{\text{Otherwise}}
\end{cases}\\
    g_{b}(\mc{X}_i,\mc{X}_j)& =\begin{cases}
1 & \text{if }\mc{X}_i\in N_{b}(\mc{X}_j)\text{ or }\mc{X}_j\in N_{b}(\mc{X}_i)\\
0 & \text{\text{Otherwise}}
\end{cases}.
\end{align*}
The set $N_w(\mc{X}_i)$ consists of $\nu_w$ nearest neighbors of $\mc{X}_i$ that have the \emph{same} labels as $y_i$, and the set $N_b(\mc{X}_i)$ consists of $\nu_b$ nearest neighbors of $\mc{X}_i$ that have \emph{different} labels from $y_i$. The nearest neighbors can be computed using the geodesic distance. 

The desired projection map $\pi_{A}$ that we seek is obtained by the minimizing the following loss function 
\[
    L_s(A) 
    %& = \frac{1}{N^2}\sum_{i, j=1}^N a(\mc{X}_i, \mc{X}_j) d^2(\pi_{\bd{A}}(\mc{X}_i), \pi_{\bd{A}}(\mc{X}_j))\\
          = \frac{1}{N^2}\sum_{i, j=1}^N a(\mc{X}_i, \mc{X}_j) d^2(\spn(A^TX_i),\spn(A^TX_j))
\]
where $d$ is a distance metric on $\gr(p, m)$. Note that if the distance metric $d$ has $\text{O}(m)$-symmetry, e.g.\ the geodesic distance, so does $L_s$. In this case the optimization can be done on $\st(m,n)/\text{O}(m) \cong \gr(m,n)$. Otherwise it is on $\st(m, n)$. This supervised dimensionality reduction is termed as, supervised nested Grassmann (sNG).

\subsection{Choice of the distance function}\label{sub:distance}

The loss functions $L_u$ and $L_s$ depend on the choice of the distance $d:\gr(p, n) \times \gr(p, n) \to \R_{\geq 0}$. Besides the geodesic distance, there are many widely used distances on the Grassmann manifold, see, for example, \citet[p.\ 337]{edelman1998geometry} and \citet[Table 2]{ye2016schubert}. In this work, we use two different distance metrics: (1) the geodesic distance $d_g$ and (2) the projection distance, which is also called the chordal distance in \cite{ye2016schubert} and the projection $F$-norm in \cite{edelman1998geometry}. The geodesic distance was defined in Section~\ref{sub:geometry} and the projection distance is defined as follows. For $\mc{X}, \mc{Y} \in \gr(p, n)$, denote the projection matrices onto $\mc{X}$ and $\mc{Y}$ by $P_{\mc{X}}$ and $P_{\mc{Y}}$ respectively. Then, the distance between $\mc{X}$ and $\mc{Y}$ is given by $d_{p}(\mc{X}, \mc{Y}) = \Vert P_{\mc{X}} - P_{\mc{Y}} \Vert_F/\sqrt{2} = \big(\sum_{i=1}^p \sin^2\theta_i\big)^{1/2}$
where $\theta_1,\ldots,\theta_p$ are the principal angles of $\mc{X}$ and $\mc{Y}$.  If $\mc{X}=\spn(X)$, then $P_{\mc{X}}=X(X^TX)^{-1}X^T$. It is also easy to see the the projection distance has $\text{O}(n)$-symmetry. We choose the projection distance mainly for its computational efficiency as it involves only matrix multiplication which has a time complexity $O(n^2)$ whereas the geodesic distance requires an SVD which has a time complexity of $O(n^3)$.

\subsection{Analysis of Principal Nested Grassmannians}\label{sub:nested_grassmann}

To determine the dimension of the nested submanifold that fits the data well enough, we can choose $p < m_1 < \ldots < m_k < n$ and estimate the projection onto these nested Grassmann manifolds. The ratio of expressed variance for each projection is the ratio of the variance of the projected data and the variance of the original data. With these ratios, we can choose the desired dimension according to the pre-specified percentage of expressed variance as one would do for choosing the number of principal components in PCA.

Alternatively, one can have a full analysis of principal nested Grassmanns (PNG) as follows. Starting from $\gr(p, n)$, one can reduce the dimension down to $\gr(p, p+1)$. Using the diffeomorphism between $\gr(p, n)$ and $\gr(p, n-p)$, we have $\gr(p, p+1) \cong \gr(1,p+1)$, and hence we can continue reducing the dimension down to $\gr(1, 2)$. The resulting sequence will be 
\[
\gr(p,n) \to \gr(p,n-1) \to \cdots \to \gr(p, p+1) = \gr(1, p+1) \to \gr(1, p) \to \cdots \to \gr(1,2).
\]
Furthermore, we can reduce the points on $\gr(1,2)$, which is a 1-dimensional manifold, to a 0-dimensional manifold, which is simply a point, by computing the Fr\'{e}chet mean (FM). We call this FM the nested Grassmannian mean (NGM) of $\mc{X}_1,\ldots,\mc{X}_N \in \gr(p,n)$. The NGM is unique since $\gr(1,2) \cong \R P^1$ can be identified as the half circle in $\R^2$ and the FM is unique in this case. Note that in general, the NGM will not be the same as the FM of $\mc{X}_1,\ldots,\mc{X}_N$ since the embedding/projection need not be isometric. The supervised PNG (sPNG) can be obtained similarly by replacing each projection with it supervised counterpart.

\subsection{Principal Scores}\label{sub:score}

Whenever we apply a projection $\pi_m: \gr(p, m+1) \to \gr(p,m)$ to the data, we might lose some information contained in the data. More specifically, since we project data on a $p(m+1-p)$-dimensional manifold to a $p(m-p)$-dimensional manifold, we need to describe this $p(m+1-p) - p(m-p) = p$ dimensional information loss during the projection. In PCA, this is done by computing the scores of each principal component, which are the transformed coordinates of each sample in the eigenspace of the covariance matrix. We can generalize the notion of principal scores to the nested Grassmanns as follows: For each $\mc{X} \in \gr(p, m+1)$, denote by $M_{\pi_m(\mc{X})}$, the fiber of $\pi_m(\mc{X})$, i.e.\ $M_{\pi_m(\mc{X})} = \pi_m^{-1}(\pi_m(\mc{X})) = \{ \mc{Y} \in \gr(p, m+1): \pi_m(\mc{Y}) = \pi_m(\mc{X})\}$, which is a $p$-dimensional submanifold of $\gr(p,m+1)$. An illustration of this fiber is given in Figure~\ref{fig:fiber}. Let $\widetilde{\mc{X}} = \iota_m(\pi_m(\mc{X}))$ and let the unit tangent vector $V \in T_{\widetilde{\mc{X}}}M_{\pi_m(\mc{X})}$ be the geodesic direction from $\widetilde{\mc{X}}$ to $\mc{X}$. Given a suitable basis on $T_{\widetilde{\mc{X}}}M_{\pi_m(\mc{X})}$, $V$ can be realized as a $p$-dimensional vector, and this will be the score vector of $\mc{X}$ associated with the projection $\pi_m$.

\begin{figure}
     \centering
     \begin{subfigure}[t]{0.45\textwidth}
         \centering
         \includegraphics[width=\textwidth]{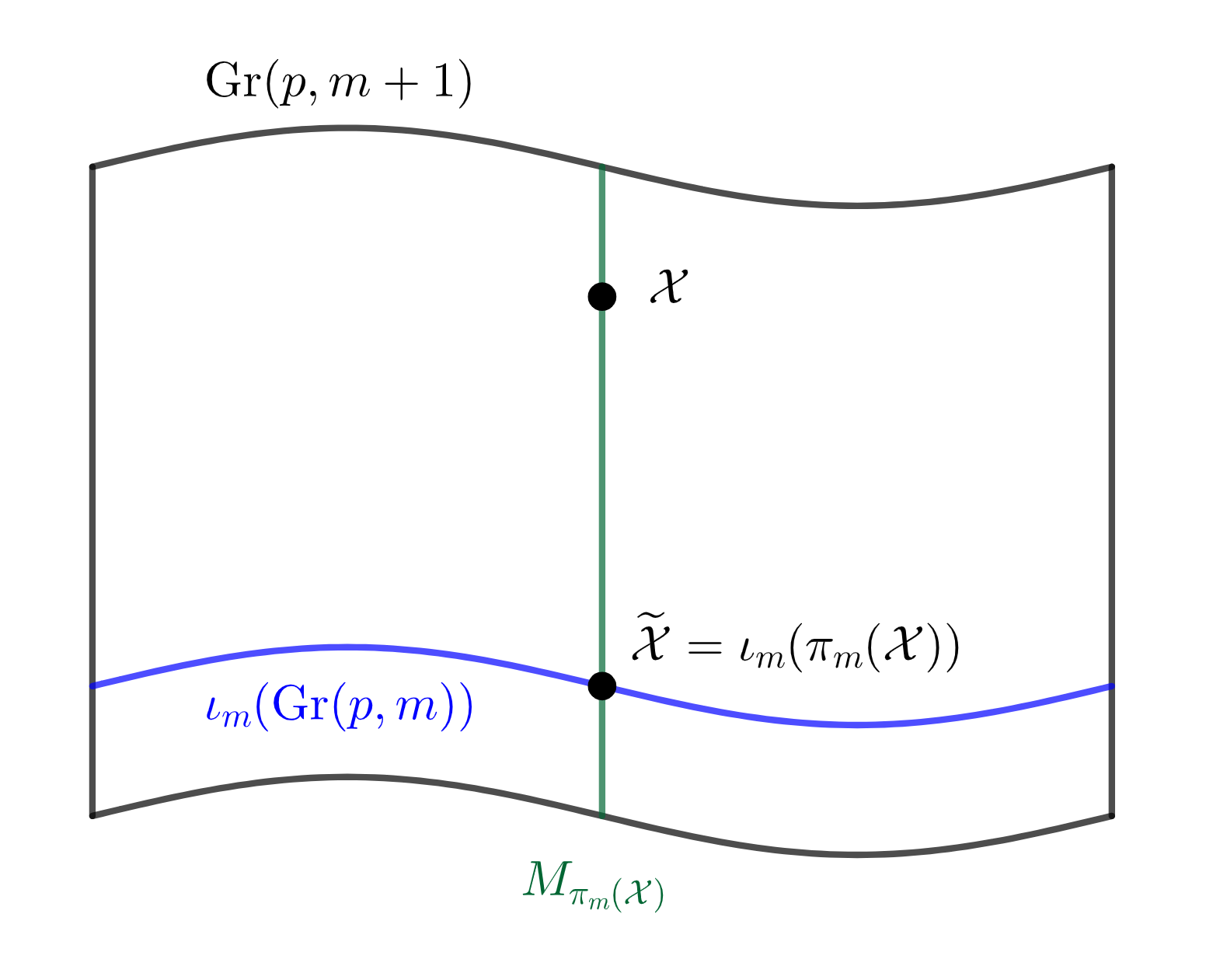}
         \caption{The fiber $M_{\pi_m(\mc{X})}$ in $\gr(p, m+1)$.}
         \label{fig:fiber}
     \end{subfigure}
     \hspace{1cm}
     \begin{subfigure}[t]{0.45\textwidth}
         \centering
         \includegraphics[width=\textwidth]{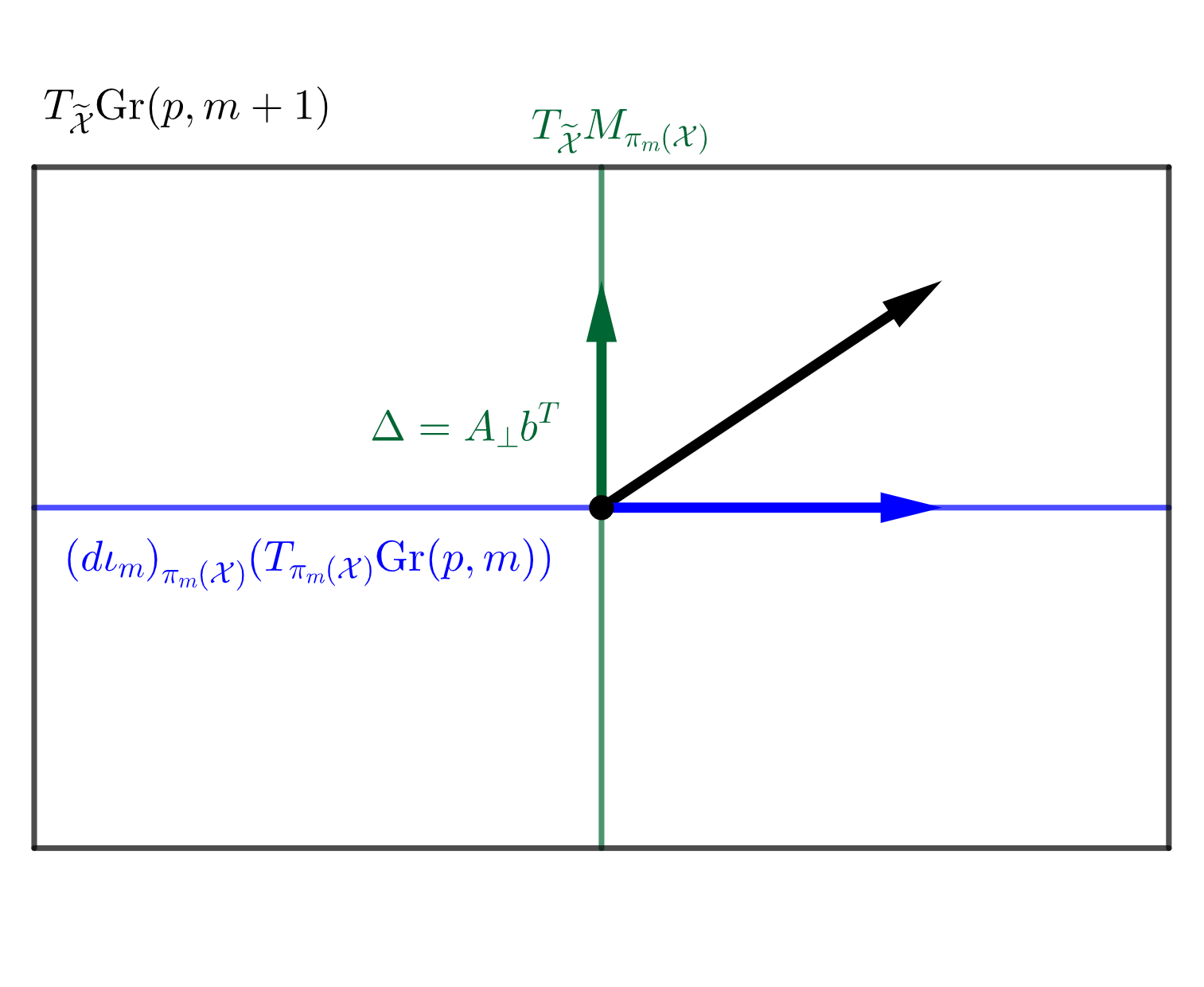}
         \caption{The horizontal space (in blue) and the vertical space (in green) induced by $\pi_m: \gr(m+1) \to \gr(p,m)$.}
         \label{fig:tangent}
     \end{subfigure}
        \caption{Illustrations of submanifolds induced by $\pi_m$.}
        \label{fig:fiber_tangent}
\end{figure}

By the definition of $M_{\pi_m(\mc{X})}$, we have the following decomposition of the tangent space of $\gr(p, m+1)$ at $\widetilde{\mc{X}}$ into the horizontal space and the vertical space induced by $\pi_m$, 
\[
T_{\widetilde{\mc{X}}}\gr(p, m+1) = T_{\widetilde{\mc{X}}}M_{\pi_m(\mc{X})} \oplus (d\iota_m)_{\pi_m(\mc{X})}(T_{\pi_m(\mc{X})}\gr(p,m)).
\]
An illustration of this decomposition is given in Figure~\ref{fig:tangent}. A tangent vector of $M_{\pi_m(\mc{X})}$ at $\widetilde{\mc{X}}$ is of the form $\Delta = A_{\perp}b^T$ where $A_{\perp}$ is any $(m+1)$-dim vector such that $[A\;A_{\perp}]$ is orthogonal and $b \in \R^p$. It is easy to check that $\pi_m(\spn(AA^TX + A_{\perp}b^T)) = \pi_m(\spn(X)) = \spn(A^TX)$. Hence a natural coordinate for the tangent vector $\Delta = A_{\perp}b^T$ is $b \in \R^p$, and the geodesic direction from $\widetilde{\mc{X}}$ to $\mc{X}$ would be $V = X^TA_{\perp}$. It is easy to see that $\|V\|_F = 1$ since $X$ has orthonormal columns. To reflect the distance between $\widetilde{\mc{X}}$ and $\mc{X}$, i.e.\ the reconstruction error, we define $d(\widetilde{\mc{X}}, \mc{X})V$ as the score vector for $\mc{X}$ associated with $\pi_m$. In the case of $\gr(1,2) \to \text{NGM}$, we use the sign distance to the NGM as the score. For complex nested Grassmanns however, the principal score associated with each projection is a $p$-dimensional complex vector. For the sake of visualization, we transform this $p$-dimensional complex vector to a $2p$-dimensional real vector. The procedure for computing the PNG and the principal scores is summarized in Algorithm~\ref{alg:png}.

\begin{remark}
Note that this definition of principal score is not intrinsic as it depends on the choice of basis. Indeed, it is impossible to choose a $p$-dimensional vector for the projection $\pi_m$ in an intrinsic way, since the only property of a map that is independent of bases is the rank of the map. A reasonable choice of basis is made by viewing the Grassmann $\gr(p,m)$ as a quotient manifold of $\st(p,m)$, which is a submanifold in $\R^{m \times p}$. This is how we define the principal score for nested Grassmanns.
\end{remark}

\begin{algorithm}
\caption{Principal Nested Grassmanns}\label{alg:png}
\SetKwInOut{KwIn}{Input}
\SetKwInOut{KwOut}{Output}
\KwIn{$\mc{X}_1, \ldots \mc{X}_N \in \gr(p, n)$}
\KwOut{An $N \times p(n-p)$ score matrix.}
Let $\mc{X} = (\mc{X}_1,\ldots, \mc{X}_N)$.\\
\tcc{$\gr(p, n) \to \gr(p, n-1) \to \cdots \to \gr(p, p+1)$}
\For{$m = n-1, \ldots, p+1$}{
    $\hat{v}, \hat{b} = \argmin_{v, b} \sum^N_{i=1}d^2(\mc{X}_i, \spn((I-vv^T)X_i+vb^T))$\; 
    \tcc{$v \in \R^{m+1}$, $\|v\| = 1$, $b \in \R^p$}
    \For{$i = 1, \ldots, N$}{
    Let $\widehat{\mc{X}}_i = \spn((I-\hat{v}\hat{v}^T)X_i+\hat{v}\hat{b}^T)$ where $\mc{X}_i = \spn(X_i)$.\\
    $V_{i,m} = d_g(\mc{X}_i,\widehat{\mc{X}}_i)X_i^Tv$ \tcp*{the score vector for $\mc{X}_i$}
    $\mc{X}_i \leftarrow \spn(R^TX_i)$ \tcp*{$R \in \st(m, m+1)$ such that $R^T\hat{v} = 0$.}
    }
}
\If{$p > 1$}{
    \For{$i = 1, \ldots, N$}{
        $\mc{X}_i \leftarrow \mc{X}_i^{\perp}$ \tcp*{$\gr(p, p+1) \cong \gr(1, p+1)$}
    }
    \tcc{$\gr(1, p+1) \to \gr(1, p) \to \cdots \to \gr(1,2)$}
    \For{$m = p, \ldots, 2$}{
        $\hat{v}, \hat{b} = \argmin_{v, b} \sum^N_{i=1}d^2(\mc{X}_i, \spn((I-vv^T)X_i+vb^T))$\; 
        \tcc{$v \in \R^{m+1}$, $\|v\| = 1$, $b \in \R^p$}
        \For{$i = 1, \ldots, N$}{
        Let $\widehat{\mc{X}}_i = \spn((I-\hat{v}\hat{v}^T)X_i+\hat{v}\hat{b}^T)$ where $\mc{X}_i = \spn(X_i)$.\\
        $V_{i,m} = d_g(\mc{X}_i,\widehat{\mc{X}}_i)X_i^Tv$ \tcp*{the score vector for $\mc{X}_i$}
        $\mc{X}_i \leftarrow \spn(R^TX_i)$ \tcp*{$R \in \st(m, m+1)$ such that $R^Tv = 0$.}
    }
    }
}
\tcc{$\gr(1,2) \to \text{NGM}$}
\tcc{Note that $\widehat{\mc{X}}_i \in \gr(1,2)$}
$\text{NGM} \leftarrow \text{FM}(\widehat{\mc{X}}_1, \ldots, \widehat{\mc{X}}_N)$\;
Let $U_i \in \R^2$ be the geodesic direction from the NGM to $\mc{X}_i$ such that $\|U_i\| = d_g(\text{NGM}, \mc{X}_i)$.\;
\For{$i = 1, \ldots, N$}{
    $V_{i,1} \in \R$ is such that $U_i = V_{i,1}\times \frac{U_1}{\|U_1\|}$\; 
}
Return the score matrix $S = [V_{i,m}]$. (Note that $V_{i,m} \in \R$ for $m = 1,\ldots,p$ and $V_{i,m} \in \R^p$ for $m = p+1, \ldots, n-1$.)
\end{algorithm}

\section{Experiments}\label{sec:exp}

In this section, we will demonstrate the performance of the proposed dimensionality reduction technique, i.e.\ PNG and sPNG, via experiments on synthetic and real data. The implementation\footnote{Our code is available at \url{https://github.com/cvgmi/NestedGrassmann}.} is based on the python library \texttt{pymanopt} \citep{townsend2016pymanopt} and we use the steepest descent algorithm for the optimization (with default parameters in pymanopt). The optimization was performed on a desktop with 3.6GHz Intel i7 processors and took about 30 seconds to converge. 
\subsection{Synthetic Data}%
\label{sub:synthetic_data}

In this subsection, we compare the performance of the projection and the geodesic distances respectively. The questions we will answer are the following. (1) From Section~\ref{sub:distance}, we see that using projection distance is more efficient than using the geodesic distance. But how do they perform compared to each other under varying dimension $n$ and variance level $\sigma^2$? (2) Is our method of dimensionality reduction "better" than PGA? Under what conditions does our method outperform PGA? 

\subsubsection{Projection and Geodesic Distance Comparisons}%
\label{sub:comparison_of_projective_f_norm_and_geodesic_distance}

The procedure we used to generate random points on $\gr(p, n)$ for the synthetic data experiments is as follows: First, we generate $N$ points from a uniform distribution on $\st(p, m)$ \cite[Ch.\ 2.5]{chikuse2003statistics}, generate $A$ from the uniform distribution on $\st(m, n)$, and generate $B$ as an $n \times p$ matrix with i.i.d entries from $N(0,0.1)$. Then we compute $\widetilde{\mc{X}}_i = \spn(AX_i + (I-AA^T)B) \in \gr(p, n)$. Finally, we compute $\mc{X}_i = \text{Exp}_{\widetilde{\mc{X}}_i}(\sigma U_i)$, where $U_i = \widetilde{U}_i/\|\widetilde{U}_i\|$ and $\widetilde{U}_i \in T_{\widetilde{\mc{X}}_i}\gr(p, n)$, to include some perturbation.

This experiment involves comparing the performance of the NG representation in terms of the explained variance, under different levels of data variance. In this experiment, we set $N=50$, $n=10$, $m=3$, and $p=1$ and $\sigma$ is ranging from 0.5 to 1. The results are averaged over $100$ repetitions and are shown in Figure~\ref{fig:v_ratio_dist}. From these results, we can see that the explained variance for the projection distance and the geodesic distance are indistinguishable but using projection distance leads to much faster convergence than when using the geodesic distance. The reason is that when two points on the Grassmann manifold are close, the geodesic distance can be well-approximated by the projection distance. When the algorithm converges, the original point $\mc{X}_i$ and the reconstructed point $\hat{\mc{X}}_i$ should be close and the geodesic distance can thus be well-approximated by the projection distance. Therefore, for all the experiments in the next section, we use the projection distance for the sake of efficiency.

\begin{figure}[H]
    \centering
    \includegraphics[width=0.5\linewidth]{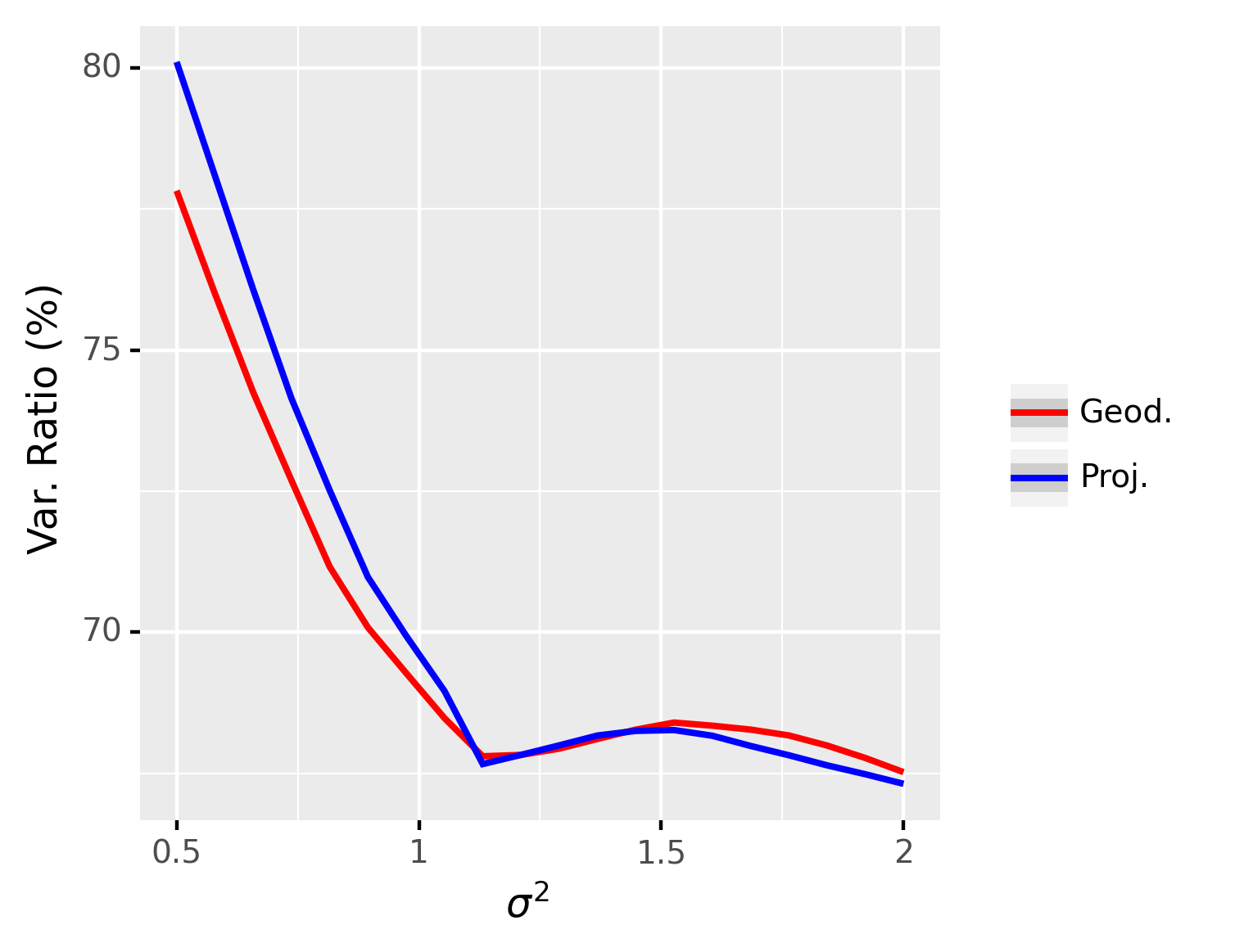}
    \caption{Comparison of the NG representations based on the projection and geodesic distances using the expressed variance.}
    \label{fig:v_ratio_dist}
\end{figure}

\subsubsection{Comparison of PNG and PGA}%
\label{sub:comparison_of_nested_grassmann_and_pga}
Now we compare the proposed PNG to PGA. In order to have a fair comparison between PNG and PGA, we define the principal components of PNG as the principal components of the scores obtained as in Section~\ref{sub:score}. Similar to the previous experiment, we set $N=50$, $n=10$, $m=5$, $p=2$, and $\sigma=0.01, 0.05, 0.1, 0.5$ and apply the same procedure to generate synthetic data. The results are averaged over $100$ repetitions and are shown in Figure~\ref{fig:v_ratio}.

\begin{figure}[H]
    \centering
    \includegraphics[width=0.6\linewidth]{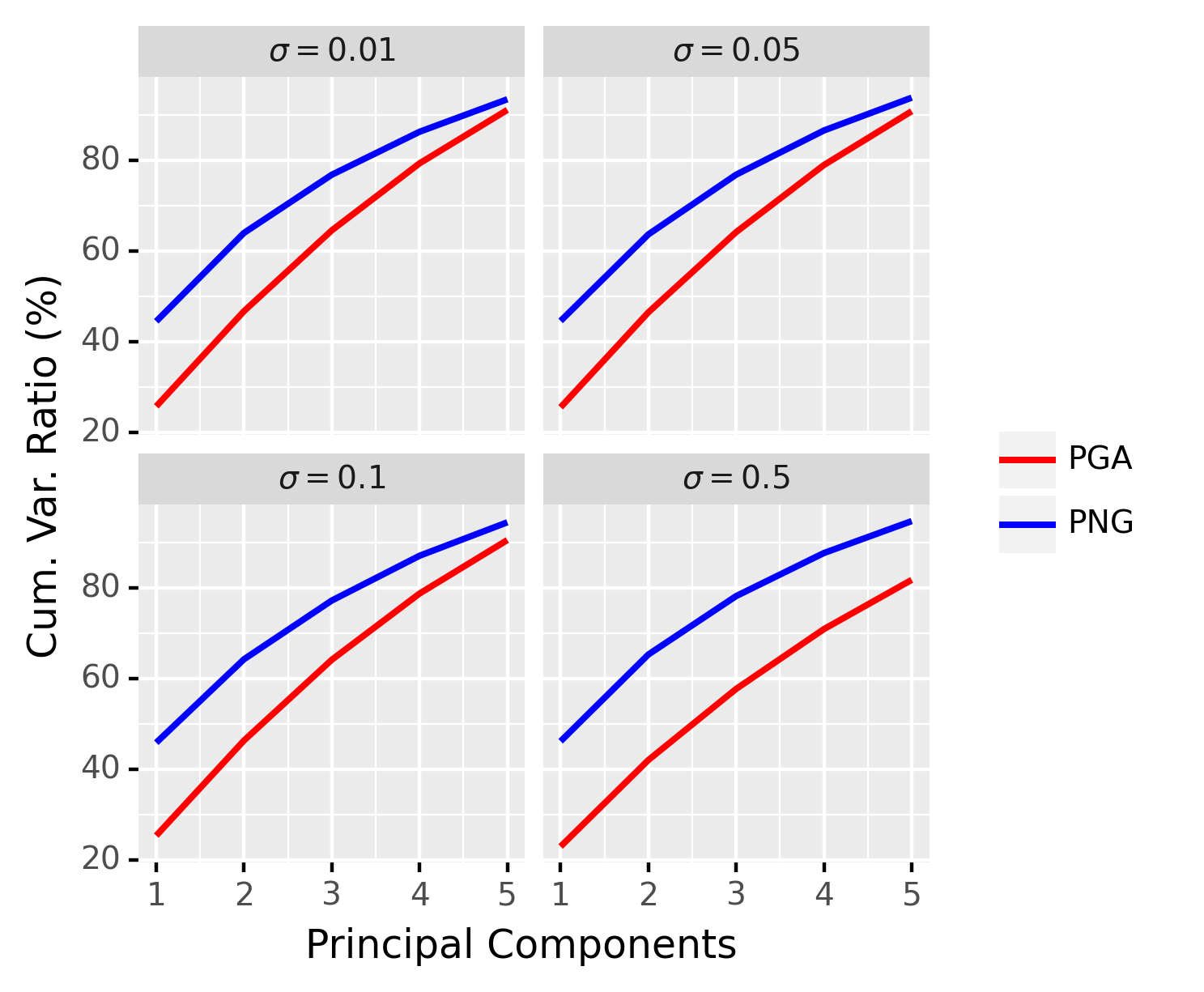}
    \caption{Comparison of PNG and PGA under different levels of noise.}
    \label{fig:v_ratio}
\end{figure}

From Figure~\ref{fig:v_ratio}, we can see that our method outperforms PGA by virtue of the fact that it is able to capture a larger amount of variance contained in the data. 
%Next, we will investigate the conditions under which our method and PGA perform equally well and when our method outperforms PGA. To answer this question, we set $N=50$, $n=10$, $m=5$, $p=2$, and $\sigma$ is ranging from 0.01 to 2. We then apply PGA and NG to reduce the dimension to 2 (i.e.\ choosing 2 principal components in PGA and project to $\gr(2, 3)$ in NG). The results are averaged over $100$ repetitions and are shown in Figure~\ref{fig:var_PGA}. 
We can see that when the variance is small, the improvement of PNG over PGA is less significant, whereas, our method is significantly better for the large data variance case (e.g.\ comparing $\sigma = 0.5$ and $\sigma=0.01$). Note that when the variance in the data is small, i.e.\ the data are tightly clustered around the FM, and PGA captures the essence of the data well. However, the requirement in PGA on the geodesic submanifold to pass through the anchor point, namely the FM, is not meaningful for data with large variance as explained through the following simple example. Consider, a few data points spread out on the equator of a sphere. The FM in this case is likely to be the north pole of the sphere if we restrict ourselves to the upper hemisphere. Thus, the geodesic submanifold computed by PGA will pass through this FM. However, what is more meaningful is a submanifold corresponding to the equator, which is what a nested spheres representation \citep{jung2012analysis} in this case yields. In similar vein, for data with large variance on a Grassmann manifold, our NG representation will yield a more meaningful representation than PGA.
\subsection{Application to Planar Shape Analysis}%
\label{sub:applications_to_planar_shape_analysis}

We now apply our method to planar (2-dimensional) shape analysis. A planar shape $\sigma$ can be represented as an ordered set of $k > 2$ points in $\R^2$, called a $k$-ad or a configuration. Here we assume that these $k$ points are not all identical. Denote the configuration by $X$ which is a $k \times 2$ matrix. Let $H$ be the $(k-1) \times k$ Helmert submatrix \citep[Ch.\ 2.5]{dryden2016statistical}. Then $Z = HX/\|HX\|_F$ is called the \emph{pre-shape} of $X$ from which the information about translation and scaling is removed. The space of all pre-shapes is called the pre-shape space, denoted by $\mc{S}^k_2$. By definition, the pre-shape space is a $(2k - 3)$-dimensional sphere. The shape is obtained by removing the rotation from the pre-shape, and thus the shape space is $\Sigma^k_2 = \mc{S}^k_2/\text{O}(2)$. It was shown by \cite{kendall1984shape} that $\Sigma^k_2$ is a smooth manifold and, when equipped with the quotient metric, is isometric to the complex projective space $\Cpx P^{k-2}$ equipped with the Fubini-Study metric (up to a scale factor) which is a special case of the complex Grassmannians, i.e.\ $\Cpx P^{k-2} \cong \gr(1,\Cpx^{k-1})$. Hence, we can apply the proposed PNG to planar shapes. For planar shapes, we also compare with the recently proposed principal nested shape spaces (PNSS) \citep{dryden2019principal}, which is an application of PNS on the pre-shape space. We will now demonstrate how the PNG performs compared to PGA and PNSS using some simple examples of planar shapes and the OASIS dataset.

\paragraph{Examples of Planar Shapes} We take three datasets: \texttt{digit3}, \texttt{gorf}, and \texttt{gorm}, from the R package \texttt{shapes}~\citep{dryden2021shapes}. The \texttt{digit3} dataset consists of 30 shapes of the digit 3, each of which is represented by 13 points in $\R^2$; the \texttt{gorf} dataset consists of 30 shapes of female gorilla skull, each of which is represented by 8 points in $\R^2$; the \texttt{gorm} dataset consists of 29 shapes of male gorilla skull, each of which is represented by 8 points in $\R^2$. Example shapes from these three datasets are shown in Figure~\ref{fig:example_shapes}. The cumulative ratios of variance explained by the first 5 principal components\footnote{Here the principal components in PNG and PGA are complex whereas the principal components in PNSS are real. Hence, we choose 5 principal components in PNG and PGA and 10 principal components in PNSS, so that the reduced (real) dimension is 10 in all three cases.} of PNG, PGA, and PNSS are shown in Figure~\ref{fig:shape_var_ratio}. It can be seen from Figure~\ref{fig:shape_var_ratio} that the proposed PNG achieves higher explained variance than PGA and PNSS respectively in most cases.

\begin{figure}
     \centering
     \begin{subfigure}[b]{0.2\textwidth}
         \centering
         \includegraphics[width=\textwidth]{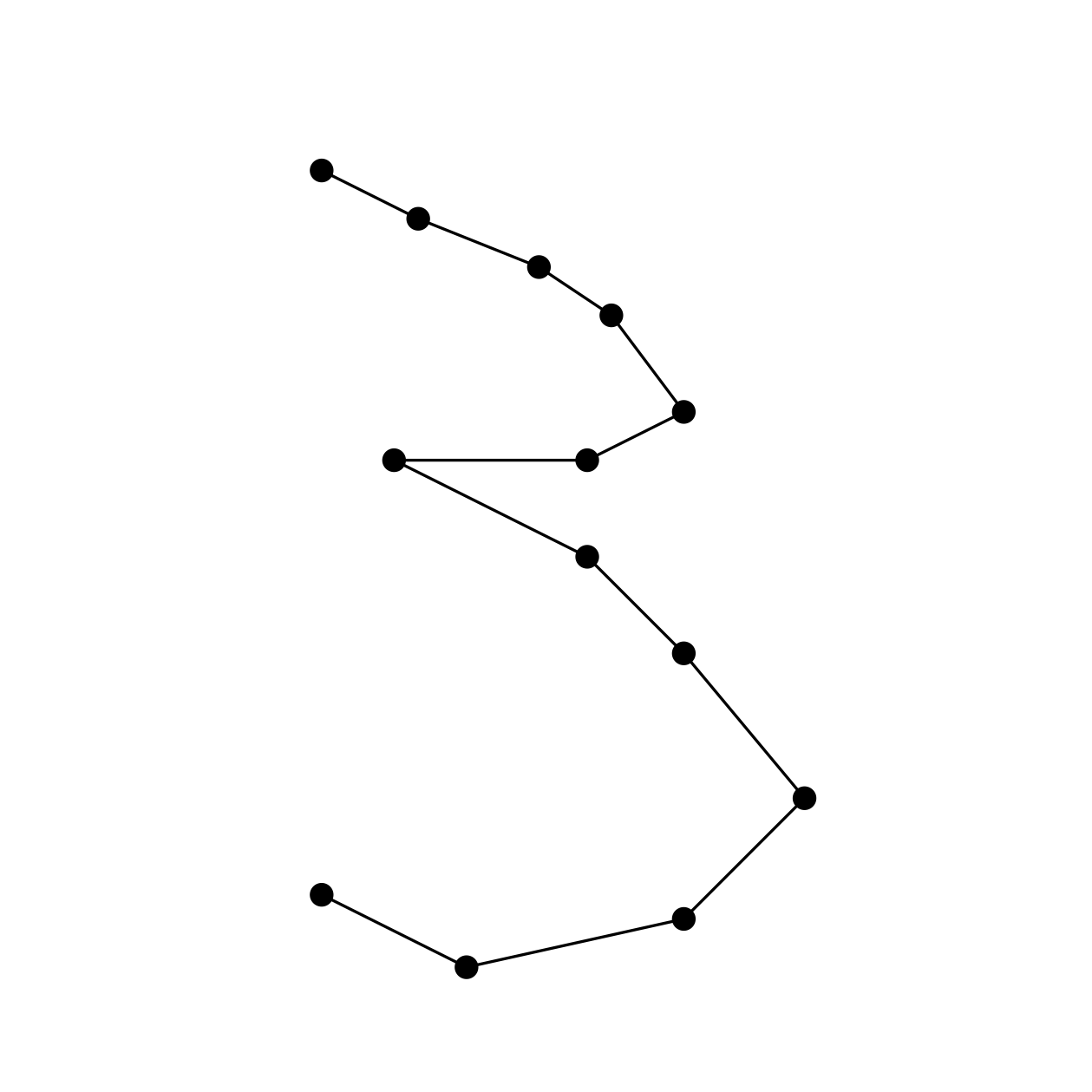}
         \caption{\texttt{digit3}}
     \end{subfigure}
     \hspace{1cm}
     \begin{subfigure}[b]{0.2\textwidth}
         \centering
         \includegraphics[width=\textwidth]{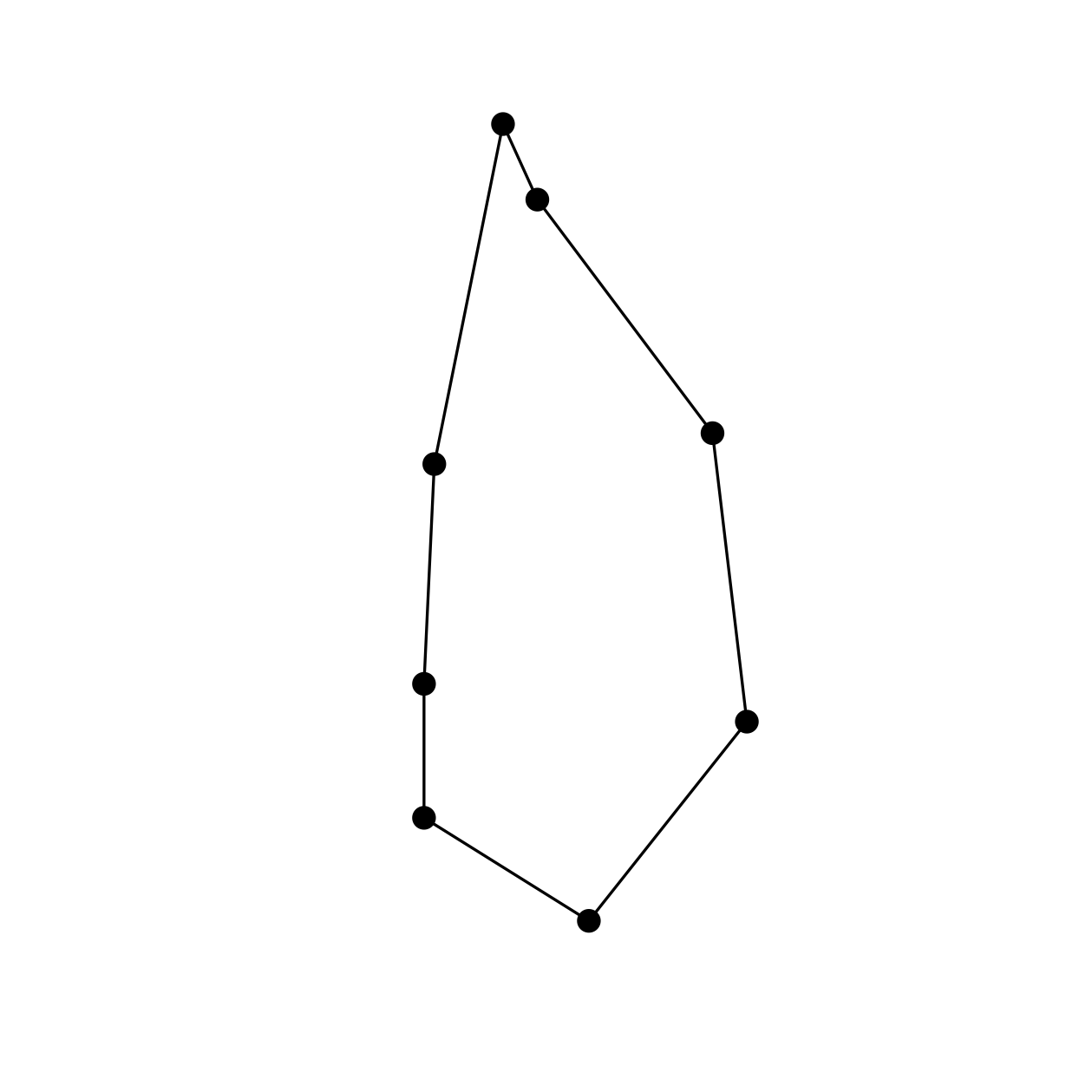}
         \caption{\texttt{gorf}}
     \end{subfigure}
     \hspace{1cm}
     \begin{subfigure}[b]{0.2\textwidth}
         \centering
         \includegraphics[width=\textwidth]{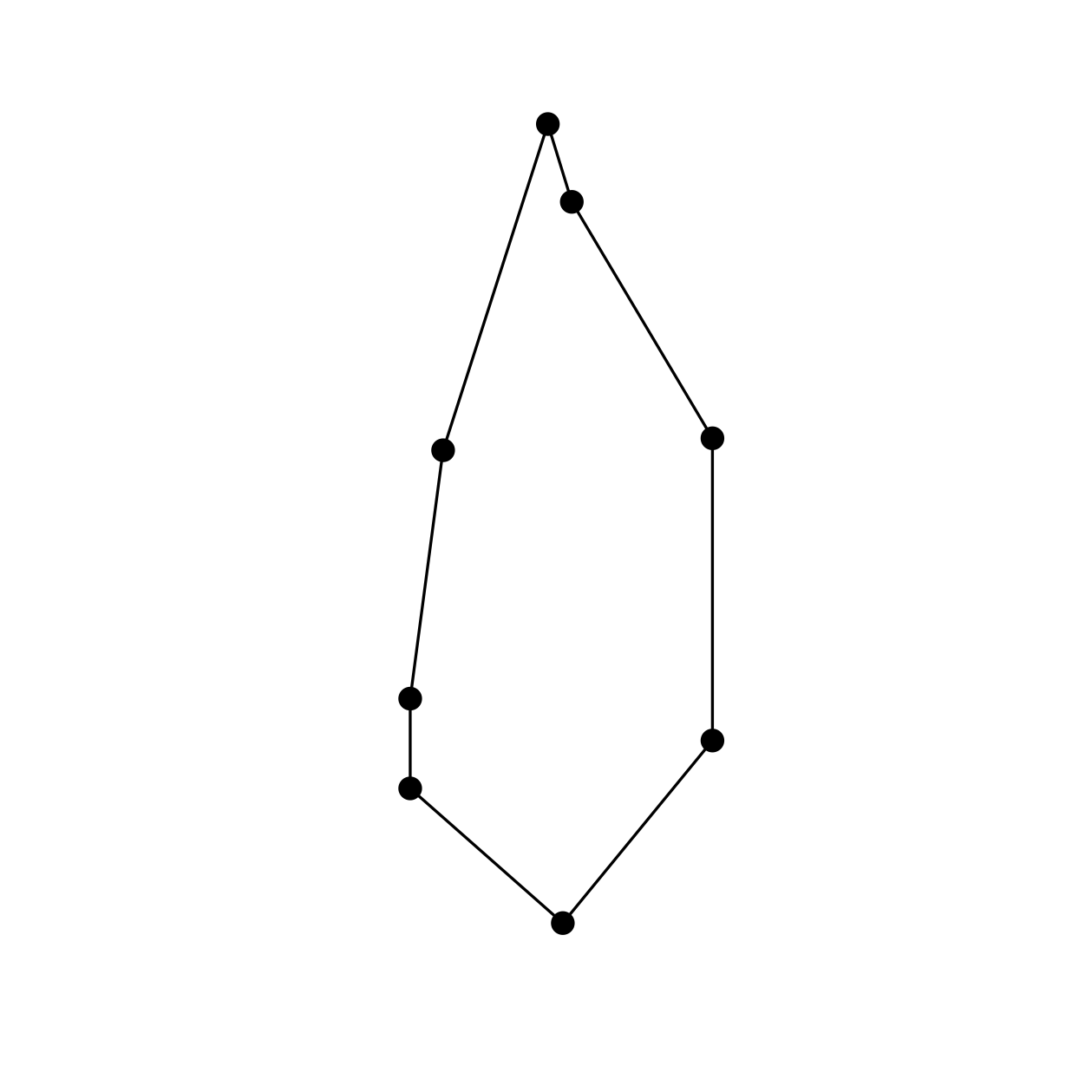}
         \caption{\texttt{gorm}}
     \end{subfigure}
        \caption{Example shapes from the three datasets.}
        \label{fig:example_shapes}
\end{figure}

\begin{figure}
    \centering
    \includegraphics[width=\textwidth]{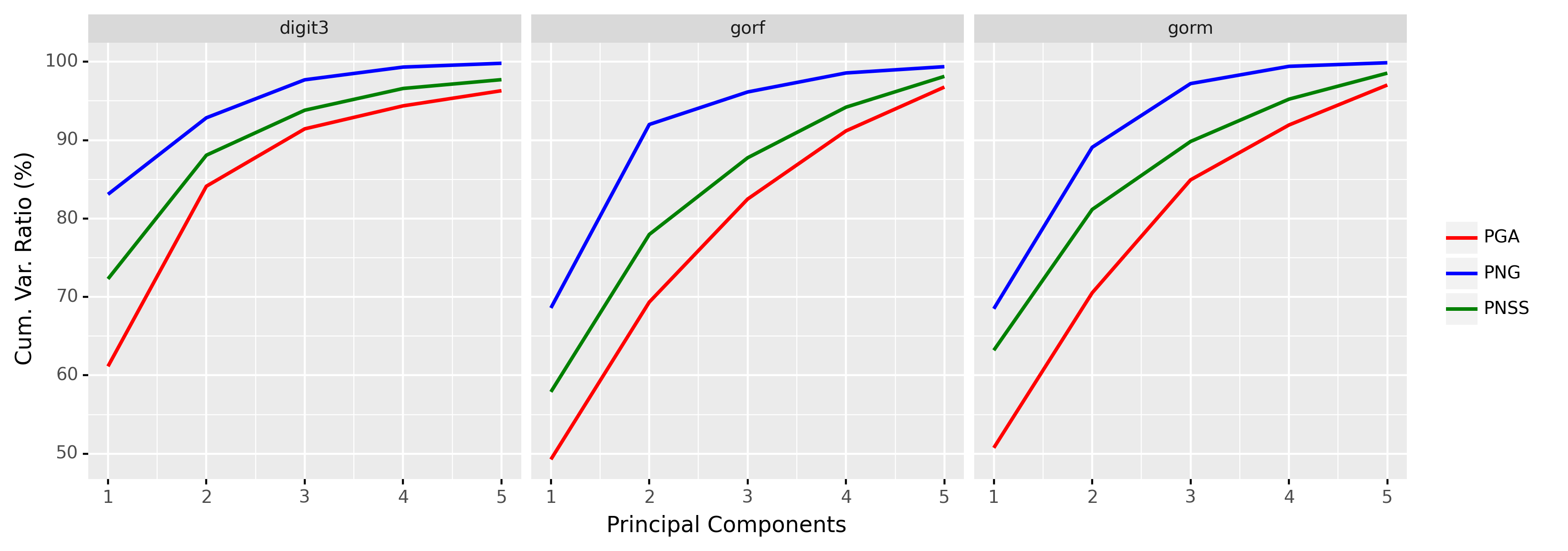}
    \caption{Cumulative explained variance by the first 5 principal components of PNG, PGA, and PNSS.}
    \label{fig:shape_var_ratio}
\end{figure}

\paragraph{OASIS Corpus Callosum Data Experiment}
The OASIS database \citep{marcus2007open} is a publicly available database that contains T1-MR brain scans of subjects of age ranging from 18 to 96. In particular, it includes subjects that are clinically diagnosed with mild to moderate Alzheimer’s disease. We further classify them into three groups: \emph{young} (aged between 10 and 40), \emph{middle-aged} (aged between 40 and 70), and \emph{old} (aged above 70). For demonstration, we randomly choose 4 brain scans within each decade, totalling 36 brain scans. From each scan, the Corpus Callosum (CC) region is segmented and 250 points are taken on the boundary of the CC region. See Figure~\ref{fig:CC} for samples of the segmented corpus callosi. In this case, the shape space is $\Sigma^{248}_2 \cong \mathbb{C}P^{248} \cong \mathsf{Gr}(1, \mathbb{C}^{249})$. Results of application of the three methods to this data are shown in Figure~\ref{fig:OASIS_result}.

\begin{figure}[H]
    \centering
    \includegraphics[width=0.6\linewidth]{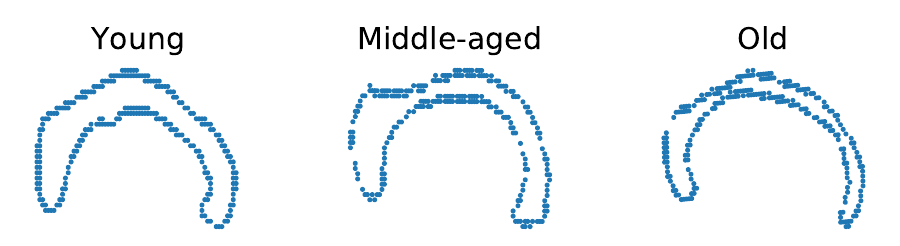}
    \caption{Example Corpus Callosi shapes from three distinct age groups, each depicted using the boundary point sets.}
    \label{fig:CC}
\end{figure}

\begin{figure}[H]
    \centering
    \includegraphics[width=0.6\linewidth]{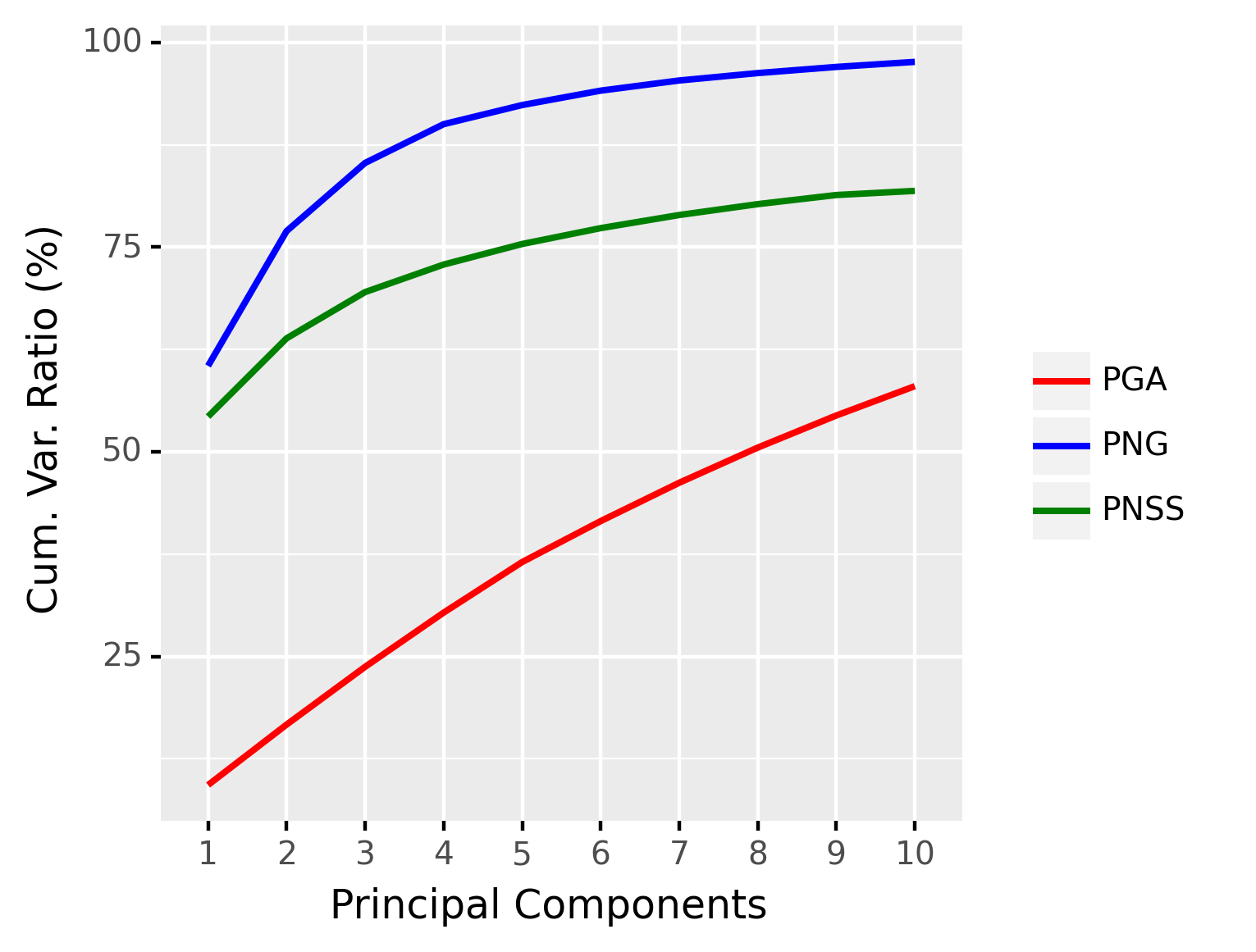}
    \caption{Cumulative explained variance captured by the first 10 principal components of PNG, PGA, and PNSS respectively.}
    \label{fig:OASIS_result}
\end{figure}

Since the data are divided into three groups (young, middle-aged, and old), we can apply the sPNG described in Section~\ref{sub:supervised_dr} to reduce the dimension. {\it The purpose of this experiment is not to demonstrate state-of-the-art classification accuracy for this dataset. Instead, our goal here is to demonstrate that the proposed nested Grassmann representation in a supervised setting is much more discriminative than the competition, namely the supervised PGA}. Hence, we choose a simple classifier such as the support vector machine (SVM)~\cite{vapnik1995nature} to highlight the aforementioned discriminative power of the nested Grassmann over PGA.

%In this experiment, for the computation of affinity matrix, we choose $\nu_w = \nu_b = 5$. 
For comparison, the PGA can be easily extended to {\it supervised PGA} (sPGA) by first diffeomorphically mapping all the data to the tangent space anchored at the FM and then performing supervised PCA \cite{bair2006prediction,barshan2011supervised} on the tangent space. However, generalizing PNSS to the supervised case is nontrivial and is beyond the scope of this paper. Therefore, we limit our comparison to the unsupervised PNSS. In this demonstration, we apply an SVM on the scores obtained from different dimension reduction algorithms, and we choose only the first three principal scores to show that even with the 3-dimensional representation of the original shapes, we can still achieve good classification results. The results are shown in Table~\ref{tab:classification}. These results are in accordance with our expectation since in sPNG, we seek a projection that minimizes the within-group variance while maximizing the between-group variance. However, as we observed earlier, the constraint of requiring the geodesic submanifold to pass through the FM is not well suited for this dataset which has a large variance across the data. This accounts for why the sPNG exhibits far superior performance compared to sPGA in accuracy.

\begin{table}[H]
\centering
\begin{tabular}{lc}
    \toprule
    Method & Accuracy\\
    \midrule
    sPNG & 83.33\% \\
    PNG	& 75\% \\
    sPGA & 66.67\% \\
    PGA	& 63.89\% \\
    PNSS & 80.56\% \\
    \bottomrule
\end{tabular}
\caption{Classification accuracies for sPGA and sPNG respectively.}
\label{tab:classification}
\end{table}

\section{Conclusion}\label{sec:conc}

In this work, we proposed a novel nested geometric structure for homogeneous spaces and used this structure to achieve dimensionality reduction for data residing in Grassmann manifolds. We also discuss how this nested geometric structure served as a natural generalization of other existing nested geometric structures in literature namely, spheres and the manifold of SPD matrices. Specifically, we showed that a lower dimensional Grassmann manifold can be embedded into a higher dimensional Grassmann manifold and via this embedding we constructed a sequence of nested Grassmann manifolds. Compared to the PGA, which is designed for general Riemannian manifolds, the proposed method can capture a higher percentage of data variance after reducing the dimensionality. This is primarily because our method, unlike the PGA, does not require the submanifold to be a geodesic submanifold and to pass through the Fr\'{e}chet mean of the data.
%The key reason behind this property is that by construction, the PGA algorithm constructs a geodesic submanifold passing through the Fr\'{e}chet mean of the data while in the nested Grassmann algorithm, there is no such constraint and the nested Grassmann representation is not necessarily a geodesic submanifold. 
Succinctly, the nested Grassmann structure allows us to fit the data to a larger class of submanifolds than PGA.
%In Euclidean space, requiring the principal subspace to pass through the sample mean is actually not a constraint since for data lying in a vector subspace, the sample mean is still in the same subspace. However, in general Riemannian manifolds, one can easily construct a counterexample that the Fr\'{e}chet mean of the data lying in a geodesic submanifold is not in the same submanifold. Hence by removing this constraint, we are able to design a better dimensionality reduction method on the Grassmann manifold. 
We also proposed a supervised dimensionality reduction technique  which simultaneously differentiates data classes while reducing dimensionality. Efficacy of our method was demonstrated on the OASIS Corpus Callosi data for dimensionality reduction and classification. We showed that our method  outperforms the widely used PGA and the recently proposed PNSS by a large margin.

%\input{appendix}

%%%%%%%%%%%%%%%%%%%%%%%%%%%%%%%%%%%%%%%%%%%%%%%%%%%%%%%%%%%%%%%%%%%%%%%
% Mandatory Sections. Please complete, especially for final publication
%%%%%%%%%%%%%%%%%%%%%%%%%%%%%%%%%%%%%%%%%%%%%%%%%%%%%%%%%%%%%%%%%%%%%%%

% Acknowledgements.
% Please include any funding, intellectual contributions not included in the authorship, and any other acknowledgements.
\acks{This research was in part funded by the NSF grant IIS-1724174, the NIH NINDS and NIA via R01NS121099 to Vemuri and the MOST grant 110-2118-M-002-005-MY3 to Yang.}

% Ethical Standards.
% Please edit with the appropriate ethics considerations for your work. Include any pertinent IRB information, etc.
%
% Please note that the submission requirements included:
% The work presented must follow appropriate ethical standards in conducting research and writing the manuscript, following all applicable laws and regulations regarding treatment of animals or human subjects.
\ethics{The work follows appropriate ethical standards in conducting research and writing the manuscript, following all applicable laws and regulations regarding treatment of animals or human subjects.}

% Conflict of Interest
% Declaration of possible conflicts of interest: Authors must disclose any financial, organisational, commercial or personal conflicts of interest that might bias their work.
% If no conflicts, please say "We declare we don't have conflicts of interest."
\coi{We declare we don't have conflicts of interest.}

\bibliography{reference}

% needed in general before appendices.
% \newpage
%\appendix % optional
%\section*{Appendix A.}

%In this appendix we prove the central theorem and present additional experimental results.
%\noindent

%{\noindent \em Remainder omitted in this sample. }

\end{document}